\documentclass[double]{IEEEtran}
\usepackage{preamble}

\begin{document}

\title{Manifold Navigation Learning Using Recurrent Networks: Efficient Training Techniques}

\author{Sarah Harkins Dayton\textsuperscript{*,1}, Sathya Rengaswami\textsuperscript{*}, Vasileios Maroulas\textsuperscript{1}, David L. Boothe\textsuperscript{2}, Ioannis D. Schizas\textsuperscript{2}
\thanks{\textsuperscript{*}S. Dayton and S. Rengaswami are co-first authors and have equally contributed to this work. Author order determined alphabetically.}
\thanks{\textsuperscript{1}Authors are with the Department of Math, University of Tennessee. Emails: \texttt{sharkin5@vols.utk.edu, vmaroula@utk.edu}}
\thanks{\textsuperscript{2}Authors are with the U.S. Army Combat Capabilities Development Command, Army Research Laboratory.}
\thanks{The work in this manuscript was support by Collaborative Agreements W911NF-22-2-0097 and W911NF-22-2-0139.}
}



\maketitle

\begin{abstract}
Position tracking is a well-studied problem in Euclidean navigation, yet comparatively little attention has been given to navigation on general manifolds, where motion is governed by the underlying geometry of the space. In this work, we formalize the problem of manifold navigation as a velocity-driven tracking task and propose a coordinate-chart–based learning framework in which learning models operate entirely in a Euclidean representation. Predicted trajectories are mapped to the manifold via an immersion function prior to error computation, allowing geometric consistency to be evaluated without explicitly enforcing manifold constraints during learning. Within this framework, we focus on Gated Recurrent Units (GRUs) for sequential position prediction on curved manifolds after comparisons with alternative recurrent architectures. We evaluate our GRU models, trained with pertinent custom loss functions, against both a standard MSE loss function and a naive vector-addition baseline. We evaluate our approach on both analytically (and numerically) defined manifolds, as well as a surface reconstructed from a LiDAR point cloud. Our results show that recurrent models can learn to approximate manifold-constrained dynamics from velocity sequences and consistently outperform Euclidean baselines. Our results demonstrate that certain custom loss functions significantly improve long-horizon prediction accuracy compared to conventional point-wise loss functions. 

\end{abstract}

\begin{IEEEkeywords}
Navigation, Recurrent Learning Architectures, Manifolds, Coordinate Chart
\end{IEEEkeywords}

\section{Introduction}

The problem of estimating the position of an object from its prior state, determined by its last known position, velocity, direction, and elapsed time, is commonly referred to as dead reckoning and has long been fundamental to navigation \cite{bowditch2002american}. In scenaria where global navigation satellite system (GNSS) data is unavailable, such as during signal outages or in GNSS-denied environments, position estimation must rely entirely on inertial navigation systems \cite{groves2008principles}. These systems inherently perform dead reckoning, accumulating error over time due to sensor noise and drift \cite{UCAM-CL-TR-696}. To mitigate this divergence, Bayesian estimation frameworks are typically employed. The Kalman Filter and its non-linear variants, such as the Extended Kalman Filter,  as well as the Particle Filter, serve as the standard for fusing noisy inertial measurements with external observations via a prediction-correction cycle \cite{kalman_filter, titterton2004strapdown}. However, the traditional formulation of these filters assumes additive  noise within a Euclidean vector space, which is ill-suited for non-Euclidean domains. When motion occurs on curved manifolds, the relationship between position and velocity is governed by the underlying geometry of the space, where curvature introduces complexities that render linear Euclidean updates geometrically inconsistent \cite{Barfoot_2017}. 


Recent advances in deep learning have significantly expanded the capabilities of navigation systems across tasks such as localization, path planning, and autonomous control \cite{localization, s19051064, path_planning, golroudbari2023recent}. Architectures including convolutional neural networks, recurrent neural networks (RNNs), Long Short-Term Memory (LSTM) networks, and temporal convolutional neural networks have been employed to integrate information from heterogeneous sensors such as GPS, Doppler velocity logs, and inertial measurement units \cite{adiuku2024advancements, golroudbari2023recent, ma2021adaptive, cohen2024inertial}. In parallel, deep reinforcement learning has been widely used for decision-making and motion planning in complex environments, enabling agents to learn navigation strategies in both indoor and outdoor settings, including structured spaces like warehouses and unstructured terrains such as aerial or underwater environments \cite{zhao2024stable, golroudbari2023recent, adiuku2024advancements, raj2024intelligent, pico2025human}.

Despite these advances, existing methods primarily focus on improving predictive accuracy or control policies within Euclidean frameworks, often neglecting the intrinsic structure of the space in which navigation occurs. While many standard techniques typically represent position in 2D map projections or ambient 3D space and focus on obstacle avoidance \cite{zhao2024stable, golroudbari2023recent, adiuku2024advancements, raj2024intelligent, pico2025human, cohen2024inertial, path_planning},  it remains unclear whether such architectures can learn meaningful internal representations of environments with nontrivial geometry. This limitation is especially critical in applications involving uneven terrains, where motion is constrained to curved manifolds. To address this, our work shifts the modeling focus to the coordinate chart representation of the manifold. By formulating the model inputs and outputs directly within these local charts for an arbitrary surface, we enable the network to account for geometric constraints, such as curvature and metric distortions, that are otherwise lost in traditional Euclidean embeddings.


Our goal in this paper is to isolate the purely dynamical aspects of navigation and investigate whether a model can learn to estimate position on a manifold given only an initial position and a sequence of velocities. In Euclidean space, this problem is trivial and boils down to repeated vector addition. However, on a manifold, position updates are governed by the geometry of the space and require the use of the exponential map, which computes geodesic flow from a given point and velocity. This operation typically involves solving a system of ordinary differential equations, making it computationally expensive and impractical for real-time applications \cite{lee2018introduction}. 


To this end, we adopt a coordinate-based formulation in which both inputs and outputs of the model are expressed in a local coordinate chart. The coordinate chart enables the proposed learning model to operate entirely in a Euclidean space, without explicitly enforcing manifold constraints during training. Instead, predicted trajectories are mapped to the manifold via an appropriate immersion function prior to error computation. This allows prediction accuracy to be evaluated in the manifold space while keeping the learning problem itself in a simpler coordinate domain.


We utilize recurrent neural networks, focusing on gated recurrent networks (GRUs), on a manifold to predict future positions of a moving entity. In addition, we introduce custom loss (location error) aggregation schemes that weigh errors across the predicted trajectory, rather than uniformly penalizing all time steps. This enables us to study how different training objectives influence long-horizon prediction accuracy and generalization.

To evaluate these questions, we construct a novel dataset consisting of trajectories on both analytically and numerically defined manifolds, as well as surfaces reconstructed from point cloud data. This includes a LiDAR-derived reconstruction of San Augustin Peak \cite{san_augustin_data}, where the manifold structure must be inferred from discrete samples. Using this framework, we systematically investigate the extent to which recurrent models can learn geometry-consistent dynamics from data, and how training design choices impact their performance.


While several works have explored learning dynamical systems on manifolds, these approaches typically rely on continuous-time models or explicit geometric constraints \cite{10574334, KODITSCHEK1990412}. To the best of our knowledge, there has been no systematic study of discrete-time, velocity-driven trajectory prediction using recurrent models trained entirely in coordinate spaces, with the manifold structure introduced only through post-hoc immersion-based evaluation. The main contributions of this work are as follows:

\begin{enumerate}


\item We introduce a coordinate-based learning framework with immersion-based error evaluation, decoupling model training from explicit geometric computations and manifold surface dependencies.

\item We propose trajectory-level loss aggregation schemes that potentially improve  performance on long-horizon position prediction tasks.

\item We provide a systematic empirical study of recurrent models on manifold motion learning across diverse manifold settings.

\end{enumerate}

The paper is organized as follows. In Section \ref{sec:background} the problem statement is outlined, while background information is given for gated recurrent units and manifold navigation. Section \ref{sec:methodology} introduces the custom loss functions, systematic subsampling of indices, and the baseline model. Section \ref{sec:datasets} focuses on the generation and purpose of the datasets being used. Section \ref{sec:results} describes the numerical experiments, and discusses the results. Section \ref{sec:conclusion} provides conclusions and future work. 

%
\section{Problem Statement and Background}\label{sec:background}

On a Euclidean surface, the task of location prediction  turns out to be trivial and corresponds to adding the velocity vectors to determine the final location. However, this is not the case on a manifold surface and we next devise learning strategies to accomplish location prediction.

\subsection{Manifold Navigation}
Suppose a particle is located at a point $p \in \mathcal{M}$ and a tangent 'velocity' vector $v \in T_{p}\mathcal{M}$, where $T_{p}\mathcal{M}$ is the plane tangent to the manifold $\mathcal{M}$ at the point $p$, is applied for one unit of time. The particle follows a geodesic, the generalization of a straight line to curved space, traveling a distance $\|v\|$ to arrive at point $q \in \mathcal{M}$. This operation is defined by the exponential map:
\begin{equation}\label{eqn:exp_map}
\text{exp}_{p}: T_{p}\mathcal{M} \to \mathcal{M}, \quad q = \text{exp}_{p}(v).
\end{equation}
In Euclidean space, Equation \eqref{eqn:exp_map} reduces to simple vector addition, $\text{exp}_{p}(v) = p + v$, as the manifold and tangent space coincide. On a general $d$-dimensional manifold $\mathcal{M} \subset \mathbb{R}^{D}$ ($D$-dimensional ambient space), however, $p+v$ typically leaves the surface of the manifold. To correctly determine the future position, we utilize a local coordinate chart $(U, \varphi)$, where $U \subset \mathbb{R}^d$ is an open set and $\varphi: U \to \mathcal{M}$ is a smooth parameterization (embedding) of $U$ in $\mathbb{R}^D$ \cite{hauberg2019bayeslearnmanifoldon}. For any point $p \in \varphi(U) \subset \mathcal{M}$, we denote its coordinate representation as 
\begin{equation}\label{eqn:coord_vec_rep}
    x = \varphi^{-1}(p) \in U.
\end{equation}

Given a coordinate position $x \in U$ and a coordinate velocity $v \in \mathbb{R}^d$, the mapping is:
\begin{equation}\label{eqn:chart_exp_map}
y=\varphi^{-1} \left( \text{exp}_{\varphi(x)}(J_{\varphi}(x)\cdot v) \right),
\end{equation} 
where $J_{\varphi}(x): \mathbb{R}^d \to T_{\varphi(x)}\mathcal{M}$ is the Jacobian of the embedding map $\phi$, which pushes the chart velocity into the manifold's tangent space. In practice, $\text{exp}_{\phi (x)}(J_{\varphi}(x)\cdot v)$ or equivalently
$\text{exp}_x(v)$ is computed by solving the corresponding geodesic ODE system directly in $U$ \cite{Pennec_ReimannStats}. 

For a sequence of coordinate velocities $\{v_t\}_{t=1}^n$, we generate a trajectory in the chart via the iterative application of:
\begin{eqnarray}\label{eqn:iterative_coord_exp}
x_{t} &=& \varphi^{-1} \left( \text{exp}_{\varphi(x_{t-1})}(J_{\varphi}(x_{t-1})\cdot v_{t}) \right), \\
 v_{t}^{\text{END}} &=& \text{PT}_{x_{t-1} \to x_t}(v_t),
\end{eqnarray}
where $v_{t}^{\text{END}}$ represents the velocity $v_{t}$ parallel transported to the tangent space $T_{x_{t}}\mathcal{M}$. The velocity $v_{t}^{\text{END}}$ is retained for the dataset because we want to ensure that the velocity at time step $t$ is in the tangent space $T_{\varphi(x_t^{\text{chart}})}\mathcal{M}$, see Section \ref{subsec:trajectories}. Parallel transport is used to move a vector, in this case a velocity vector, from one tangent plane to another along a specific curve of the manifold \cite{lee2018introduction}. Thus, the vector remains tangent to the manifold. For a complete introduction to Riemannian geometry, refer to \cite{lee2018introduction}.

\subsection{Problem Statement}\label{sec:problem_stmt}

Let $p_{0} \in \mathcal{M}$ be an initial position and  $v_{1}, \dots, v_{n}$ be a sequence of velocities on a Riemannian manifold $(\mathcal{M}, g)$ and $g$ is a choice of Riemannian metric on $\mathcal{M}$ that is a symmetric bilinear form $g: T_{p}\mathcal{M} \times T_{p}\mathcal{M} \to \mathbb{R}$ for $p \in \mathcal{M}$, where each $v_{t}\in T_{p_{t-1}}\mathcal{M}$ with $T_{p_{t-1}}\mathcal{M}$ indicating the plane tangent to the manifold $\mathcal{M}$ at the point $p_{t-1}$. This sequence defines a ground-truth trajectory $\{ p_{t}\}^n_{t=1}$ via iterative application of the exponential map as described in Eq. \eqref{eqn:exp_map}.

To enable numerical learning, we represent this system in the coordinate chart $U$. Let $x_{0} = \varphi^{-1}(p_{0})\in U$ be the initial coordinate chart position.  We define a corresponding sequence of velocities $v^{chart}_{1}, \dots, v^{chart}_{n} \in \mathbb{R}^d$, which are related to the manifold tangent velocities via the Jacobian of the embedding function $\phi$, i.e.,:
\begin{equation}
    v_t = J_\varphi(x_{t-1}) \cdot v_t^{\text{chart}}.
\end{equation}

We seek to train the parameters $\mathbb{\theta}$ of a recurrent neural network model $f_{\theta}$ to predict the sequence of coordinate chart positions $x_{1}, \dots , x_{t}$ based on the initial chart position $x_{0}$ and the history of coordinate velocities, i.e.,
\begin{equation}\label{eqn:revised_prob_stmt}
    \hat{x}_{t} = f_{\theta}(x_{0}, [v^{chart}_{1}, \dots, v^{chart}_{t}]).
\end{equation}
The predicted positions on the manifold trajectory will be obtained by mapping the predicted coordinate chart positions back to the manifold via the embedding function, 
\begin{equation}
    \hat{p}_{t} = \varphi(\hat{x}_{t}),
\end{equation}
where $\varphi$ is the immersion function defined in Eq. \eqref{eqn:coord_vec_rep}. Idelaly, the recurrent learning architecture should be trained to minimize the accumulated geodesic error between the ground-truth positions $p_{t}$ and the predicted positions $\hat{p}_{t}$:
\begin{equation}
    E = \sqrt{\frac{1}{n}\sum_{t=1}^n d_{\mathcal{M}}(p_t, \hat{p}_t)^2},
\end{equation}
where $d_{\mathcal{M}}(p,q)$ denotes geodesic distance, that is, the length of the shortest path between points $p$ and $q$ as determined by the metric $g$ [cf. Apdx. \ref{appendix:rieman_submanifold}, for details on metric $g$].  However, given the manifold geomtric structure and metric are unknown we will resort to geodesic distance approximations as discussed later on.

\subsection{Gated Recurrent Units}
To address the coordinate chart position prediction task we will build upon Gated Recurrent Units (GRUs) learning architectures. GRUs are recurrent neural network (RNN) architectures designed to improve the modeling of sequential data while addressing the vanishing gradient problem that limits standard RNNs. Introduced in \cite{cho2014learningphraserepresentationsusing}, GRUs simplify the gating mechanisms of long short-term memory (LSTM) units, another RNN architecture, by combining the forget and input gates into a single update gate and introducing a reset gate, thus streamlining the cell structure without sacrificing the ability to capture long-term dependencies. 

At each time step a new input velocity vector $v_{t}$ enters the GRU structure, while the hidden state is updated according to a gating mechanism \cite{cho2014learningphraserepresentationsusing}. First, an update gate, $z_{t}$, determines how much of the past information is carried forward. The reset gate, $r_{t}$, controls the integration of new input with past memory, allowing the GRU to adaptively forget irrelevant information  \cite{cho2014learningphraserepresentationsusing, chung2014empiricalevaluationgatedrecurrent}. Next, a candidate hidden state, $\tilde{h}_{t}$, is computed. Lastly, the final hidden state, $h_{t}$, is computed. The recursive updating formulas are as follows:
\begin{align}
    &\text{Update Gate: }  z_{t} = \sigma_{z}(W_{z}v_{t} + U_{z}h_{t-1} + b_{z} ), \\
    &\text{Reset Gate: }  r_{t} = \sigma_{r}(W_{r}v_{t} + U_{r}h_{t-1} + b_{r} ),\\
    &\text{Hid. Cand.: } \tilde{h}_{t} = tanh(W_{h} v_{t} + U_{h}(r_{t} \odot h_{t-1}) + b_{h}),\\
    &\text{Hidden State: } h_{t} = (1-z_{t}) \odot h_{t-1} + z_{t} \odot \tilde{h}_{t},
\end{align} 
where $\sigma_{*}$ denotes a nonlinear activation function for the respective layer represented by $*$, $W_{*}$ and $U_{*}$ denote model weights, $b_{*}$ denotes bias vectors, and $tanh$ denotes the hyperbolic tangent nonlinear activation function.

GRUs have been shown to perform comparably to or better than LSTMs across a range of sequence modeling tasks, including speech recognition, machine translation, and time-series forecasting, particularly in settings where computational efficiency or limited data are important considerations \cite{chung2014empiricalevaluationgatedrecurrent, yin2017comparativestudycnnrnn}. In addition, GRU-based models have been successfully applied to trajectory prediction problems such as traffic flow estimation and flight path forecasting \cite{guan2024checkpoint, jiang2019trajectory, han2019realtime, han2021combined}, where the goal is to model the evolution of a system over time based on sequential observations.

These applications are closely related to navigation tasks, in which an agent’s state evolves according to underlying dynamics that may be partially observed or subject to noise. In such settings, GRUs provide a natural framework for learning temporal dependencies and propagating state information forward in time, enabling the prediction of future positions from past states and control inputs. This makes them particularly well-suited for predicting trajectories on manifolds, where the evolution of the system is constrained by nonlinear geometric structure and where accurate long-horizon prediction requires maintaining consistency with the underlying dynamics. Furthermore, we will show via extensive numerical comparisons that GRUs offer a performance edge over alternative recursive architectures especially in vehicle-like trajectory prediction.

\section{Methodology}\label{sec:methodology}



We employ a GRU network with a single hidden layer of dimension 32 [trying lower dimensionality degraded performance while increasing it did not provide any benefits] as the learning architecture to predict the position of a moving entity. The network takes as input a sequence of observed velocities on the coordinate chart. The initial position of the entity is used to initialize the GRU's hidden state. The model outputs the predicted position of the traveling entity on the coordinate chart. To improve training, custom loss functions are constructed that are discussed in Section \ref{sec:loss_fns}. 

\subsection{Loss Functions}\label{sec:loss_fns}
The error between a ground-truth trajectory $\{p_t\}_{t=1}^n \in \mathcal{M}$ and the corresponding predicted trajectory $\{\hat{p}_t\}_{t=1}^n \in \mathcal{M}$ can be quantified accurately by the mean geodesic distance:
\begin{equation}\label{eqn:std_loss_fn_dM}
    \mathcal{L}_n^{\mathcal{M}} = \frac{1}{n}\sum_{t=1}^n [d_{\mathcal{M}}(p_t,\hat{p}_t)]^2,
\end{equation}
where $d_{\mathcal{M}}(\cdot,\cdot)$ provides a rigorous geometric measure on the manifold surface ${\cal M}$. \textcolor{black}{However, finding the geodesic distance mandates integration and global optimization operations, as defined in Eq. \eqref{eqn:dist_metric} in Appendix \ref{appendix:rieman_submanifold}, which is computationally expensive to solve or approximate \cite{crane2020surveyalgorithmsgeodesicpaths}.} To address this challenge, Euclidean distance can serve as a good proxy for manifold distances on small enough scales \cite{Pennec_ReimannStats}. This is justified in a quantitative sense in the following theoretical result [proof in Apdx. \ref{appendix:proof}].

\begin{theorem}\label{thm:norm_equiv}
    Let $\M$ be a surface whose principal curvatures are bounded by $K > 0$ in absolute value. Then, there exist constants $r>0$ and $C>0$ such that for any two points $p,q \in \mathcal{M}$ with $d_\M(p,q)<r$,

\begin{equation}\label{eq:squeeze}
    \|p-q\|_2 \leq d_{\mathcal{M}}(p,q) \le C \|p-q\|_2.
\end{equation}
\end{theorem}

{\textbf{Remark}: Thm. 1 essentially states that as long as there exists $r>0$ s.t. $\|p_t - \hat{p}_t\|^2_2 < r$, where $p_t$ is the ground truth and $\hat{p}_t = \varphi(x_t)$, is the predicted position and the model is learning from the imposed training loss function (i.e. the training loss has a decreasing trend across epochs and stays below a threshold $r$), then the Euclidean distance is an appropriate proxy for the geodesic distance $d_{\mathcal{M}}(\cdot,\cdot)$. This is ensured by `squeezing' the manifold distance with an upper and lower bound that depends on the Euclidean distance. Thus, during training in which back propagation is pushing the Euclidean distance towards zero, the manifold distance inevitably should also be decreased.  Thm. 1 essentially indicates that minimization of a Euclidean norm based training loss ensures that also the geodesic distance $d_{\mathcal{M}}(\cdot,\cdot)$ will be decreasing as a direct consequence of Eq. \eqref{eq:squeeze} and the pinching theorem in \cite{rudin1976principles}.} Therefore, 
we utilize the Euclidean Mean Squared Error (MSE) as a computationally efficient surrogate, see Apdx. \ref{appendix:proof}:
\begin{equation}\label{eqn:std_loss_fn}
\mathcal{L}_n = \frac{1}{n}\sum_{t=1}^n \|p_t - \hat{p}_t\|^2_2.
\end{equation}

\subsection{Systematic Subsampling}
Towards learning scalability for trajectories of arbitrary length, we will explore  if generalizability can be achieved without penalizing the model at every timestep. We define a time index set $I=\{i_1, i_2, \ldots, i_m\} \subset \{1, \ldots, n\}$ and consider loss functions which are subsampled versions of \eqref{eqn:std_loss_fn} having the following general form:

\begin{equation}\label{eqn:custom_loss}
    \mathcal{L}_I = \frac{1}{|I|}\sum_{i \in I} \|p_i-\hat{p}_i\|^2_2.
\end{equation}
The main idea behind \eqref{eqn:custom_loss} is that by explicitly minimizing model prediction error on a sufficiently dense subset $I$, the error at the remaining indices is implicitly minimized. The strategies for defining systematically subsampled indices $I$ are given in Table \ref{tab:subsampling}. The explicit definitions are provided in Table \ref{tab:subsampling_explicit}, where $n_{max}$ is the maximum number of time steps in a trajectory that will be considered in the loss function. Note that the standard strategy refers to the training loss provided in \eqref{eqn:std_loss_fn} where all the time indices are taken into consideration. The exponential sampling strategy consists of indices that are initially dense and then become sparse as the time index value increases. The mixed linear spacing strategy provides indices that are dense early on and become sparse later in a linear fashion, unlike the exponential strategy. \textcolor{black}{In the mixed linear spacing strategy, the $p$ term represents the number of consecutive terms which are included in the loss function (starting at the first index), and the term $q$ represents the number of linearly interpolated terms which are included from index $p+1$ to $n_{max}$.} The strategies are chosen to explore whether it is more important to consider more samples early or later on in a trajectory when learning to determine position from manifold/chart velocities. 



\begin{table}[h]
\centering
\caption{Strategies for Index Selection}
\label{tab:subsampling}
\begin{tabular}{ccl}
\toprule
\textbf{Strategy} & \textbf{Notation} & \textbf{Formula for $j^{th}$ term} \\
\midrule
Standard & $\mathcal{L}_{k}$ & $j, \quad j \in [1, k]$ \\
\hline \\
Exponential & \texttt{exp\_k} & $\text{integer}(n_{\max}^{j/(k+1)})$ \\
\hline \\
\makecell{Mixed\\linear} & 
\makecell{\texttt{mixed\_p\_q} \\ {\footnotesize $\Delta = \left\lceil\frac{n_{\max}-p}{q}\right\rceil$}} & 
$\begin{cases} 
    j, & j \in [1, p] \\
    p + \Delta(j-p), & j \in (p, p+q]
\end{cases}$ \\
\bottomrule
\end{tabular}
\end{table}


\section{Datasets}\label{sec:datasets}
Next, details are provided on how the trajectories are generated for analytically described manifolds (either analytically or by numerically solving the ODEs involved), as well as reconstructed surfaces obtained from LiDAR point clouds. The datasets are generated to capture the geometric properties involved in manifold movement and they will be subsequently used to train the recurrent architectures in the task of manifold position prediction. \textcolor{black}{To the best of our knowledge, this is the first systematic study effort to generate such datasets for testing under the realm of  manifold navigation learning.}

\subsection{Manifold Trajectories Modeling}\label{subsec:trajectories}
We construct datasets comprising of trajectories on manifolds by iteratively sampling velocities and updating positions via the exponential map [cf. \eqref{eqn:exp_map}]. While trajectory generation is defined using geometric operations on the manifold $\mathcal{M}$, all data are ultimately represented in a coordinate chart. Specifically, each trajectory is given as a sequence $\{x_t\}_{t=0}^{n} \subset \mathbb{R}^d$, where $d$ is the dimension of the chart. This representation aligns with the learning framework, in which both inputs and outputs of the model are expressed in Euclidean chart coordinates.

We consider two classes of motion: Brownian motion and autoregressive processes of order one, i.e., AR(1). These models provide complementary motion regimes, challenging the learning process with unstructured motion patterns and temporally correlated dynamics that more closely match vehicle motion. Brownian trajectories correspond to random walks on the manifold with no preferred direction. To introduce temporal correlation in motion, which better models vehicle movement in practical scenaria \cite{jing2017}, we consider an AR(1)-type process on velocity generation.

 Given an initial point $x_0^{\text{chart}} \in \mathbb{R}^d$ corresponding to a coordinate representation of a point on manifold $\mathcal{M}$, and a parameter $v_{\text{scale}} > 0$, velocities and positions are generated at each time step. These velocities represent tangent vectors under the chosen coordinate parameterization. 
 
 The Brownian motion dataset is generated according to Alg. \ref{alg:brownian}, where velocities at each time step are stochastically independent. The AR(1) dataset is generated according to Alg. \ref{alg:ar1}, where the velocity at each time step depends on the previous one. The AR(1) model parameter utilized here is set to 
 $\theta = 1$.  

\begin{algorithm}
\caption{Brownian Manifold Trajectories Generation}\label{alg:brownian}
\begin{algorithmic}[1] 

\While{$x_{t}^{\text{proposed}} \notin  U$}
\State $\omega_{t} \sim  Unif(\{v \in \mathbb{R}^{d}: \|v\| \leq v_{scale} \})$
\State $v_t^{\text{chart}} = J_\varphi^{\dagger}(x_{t-1}^{\text{chart}}) Q_{t-1}\omega_{t}$
\State $x_{t}^{\text{proposed}} = x_{t-1}^{\text{chart}} + v_t^{\text{chart}}$
\EndWhile
\State $x^{\text{chart}}_{t}   = \text{exp}_{x_{t-1}^{\text{chart}}}( v_t^{\text{chart}})$ \label{eqn:exp_brownian}
\State $v_{t}^{\text{END}} = \text{PT}_{\varphi(x_{t-1}^{\text{chart}}) \to \varphi(x_t^{\text{chart}})}(v_t^{\text{chart}})$ \label{eqn:pt_brownian}
\end{algorithmic}
\end{algorithm}

\begin{algorithm}
\caption{AR(1) Manifold Trajectories Generation} \label{alg:ar1}
\begin{algorithmic}[1] 
\While{$x_{t}^{\text{proposed}} \notin  U$}
\State $\epsilon_{t} \sim  Unif(\{v \in \mathbb{R}^{d}: \|v\| \leq v_{scale} \})$
\State $v_t^{\text{chart}} = v^{ \text{END} }_{t-1}+ J_\varphi^{\dagger}(x_{t-1}^{\text{chart}}) Q_{t-1}\epsilon_{t}\label{eqn:add_noise_ar1}$
\State $x_{t}^{\text{proposed}} = x_{t-1}^{\text{chart}} + v_t^{\text{chart}}$
\EndWhile
\State $x^{\text{chart}}_{t}   = \text{exp}_{x_{t-1}^{\text{chart}}}( v_t^{\text{chart}})$ \label{eqn:exp_ar1}
\State $v_{t}^{\text{END}} = \text{PT}_{\varphi(x_{t-1}^{\text{chart}}) \to \varphi(x_t^{\text{chart}})}(v_t^{\text{chart}})$ \label{eqn:pt_ar1}
\end{algorithmic}
\end{algorithm}

In Algorithms \ref{alg:brownian} and \ref{alg:ar1}, $U$ represents the coordinate chart, $\varphi$ denotes the embedding function, $J_\varphi^{\dagger}(x_{t-1}^{chart})$ is the Moore-Penrose pseudoinverse of the Jacobian of the embedding function $\varphi$, $Q_{t-1}$ is the resulting matrix from a reduced $QR$-decomposition of Jacobian $J_\varphi(x_{t-1}^{chart})$, $\exp$ is the exponential map, and $\text{PT}$ represents parallel transport. Further, $\omega \sim Unif(\{v \in \mathbb{R}^{d}: \|v\| \leq v_{scale} \}$ indicates uniform random sampling of a point on a disc of radius $\|v_{scale}\|$.

The matrix $Q_{t-1}$ is utilized to put the sampled velocity $\omega_{t}$ on the plane tangent to manifold point $\varphi(x_{t-1}^{chart})$. \textcolor{black}{This step is especially important for the AR(1) trajectories since we are adding $v^{END}_{t-1}$ to the sampled noise term $\epsilon_{t}$ after being mapped to the coordinate chart}. We project the velocity back into the coordinate space via $J_\varphi^{\dagger}(x_{t-1}^{chart})$, yielding $v_{t}^{chart}$. The intermediate vector $x_{t}^{\text{proposed}}$ is computed and compared against the coordinate chart to restrict the size of the velocity on the coordinate chart at each time step. If $x_{t}^{\text{proposed}} \notin U$, then the vector $\omega_{t}$ or $\epsilon_{t}$ is resampled in the respective algorithm. The exponential map sends the tangent velocity vector $v_{t}^{chart}$ at the point $x_{t-1}^{chart}$ to a point on the coordinate chart that corresponds to a point on the manifold that follows the geodesic. Parallel transport moves the velocity vector from the tangent space $T_{\varphi(x_{t-1}^{\text{chart}})}\mathcal{M}$ to $T_{\varphi(x_t^{\text{chart}})}\mathcal{M}$. The velocity $v_{t}^{\text{END}}$ is retained for use in the dataset instead of $v_{t}^{chart}$ because we want the velocity associated with $x_{t}$ to be in \textcolor{black}{$T_{\varphi(x_t)}\mathcal{M}$}. This ensures that the velocity vector remains tangent to the manifold while incorporating temporal AR(1) correlations. In Alg. 1 $v_t^{END}$ is not utilized since the velocity are randomly and independently sampled at each time instant.

While the generative logic remains consistent across datasets for both the Brownian and AR(1) trajectories, in Eqns. \eqref{eqn:exp_brownian} and \eqref{eqn:pt_brownian} of both Algs. \ref{alg:brownian} and \ref{alg:ar1}, the exponential map is computed according to the specific geometry. On the Euclidean plane, the exponential map is simply vector addition. For the sphere, explicit equations are known for both the exponential map and parallel transport along its surface \cite{lee2018introduction}. For the torus and the reconstructed mountain surface, the geodesic equation initial value problem is approximately solved with a first-order ordinary differential equation numerical solver \textcolor{black}{utilizing function odeint from the torchdiffeq python library. This involves solving the Levi-Civita connection system of ODEs and numerically calculating the required Christoffel symbols \cite{lee2018introduction}}. In effect, this approximates $x_{t}^{chart}$ and $v_{t}^{END}$ for the torus and reconstructed mountain surface. Essentially Algs. 1 and 2 uniformly sample a random vector $\omega_t\in\mathbb{R}^d$ which then map in the coordinate space to evaluate the next position which can then be mapped into a point on the manifold surface generating a trajectory.

\subsection{Manifolds for Navigation}
We evaluate trajectory prediction for the two motion classes outlined earlier across manifolds of increasing geometric complexity. These manifolds are:

\begin{enumerate}
    \item \textbf{Euclidean plane:} A square domain $(0,\pi)\times(0,\pi) \subset \mathbb{R}^2$, serving as a baseline with zero curvature.
    
    \item \textbf{Sphere:} The lower (southern) hemisphere of a sphere of radius $\sqrt{\pi}$, denoted $\mathbb{S}^2_{\sqrt{\pi},-} =  \mathbb{S}^2 \cap \{ z \leq 0\}\subset \mathbb{R}^3$, where $\mathbb{S}^{2} = \{(x,y,z)\in \mathbb{R}^{3}: x^{2}+y^{2}+z^{2} = \pi \} $ which has constant positive curvature.
    
    \item \textbf{Torus:} A torus $\mathbb{T}^2$ with major radius $c=1$ and minor radius $a=\frac{1}{4}$, exhibiting spatially varying curvature, including regions of negative curvature.
    
    \item \textbf{Reconstructed surface:} A surface representing San Augustin Peak in New Mexico, reconstructed from the United States Geological Survey LiDAR point cloud data \cite{san_augustin_data}. The surface is modeled using a 2-layer fully connected multilayer perceptron (MLP) with a hidden dimension of size 32 for both layers and $tanh$ nonlinear activation functions. The MLP takes in the x- and y-coordinates of a position on the mountain surface and is trained to predict the z-coordinate (elevation). This MLP provides \textcolor{black}{a differentiable representation of the mountain surface} (elevation) with respect to the x- and y- coordinates which are further used to numerically evaluate geometric quantities such as exponential maps and parallel transport required to the trajectories generation. \textcolor{black}{The MLP utilized earlier will act as the embedding function for the reconstructed mountain dataset.} 
\end{enumerate}


\subsection{Dataset Construction}

For each of the aforementioned manifolds, we generate the Brownian motion dataset and the AR(1) dataset. Each dataset consists of $B=100,000$ trajectories evolving in a time horizon of length $n = 100$. The velocity scale is set to $v_{\text{scale}} = 0.1$ for the Brownian motion datasets, and $v_{\text{scale}} = 0.01$ for the AR(1) datasets. For model training and testing, an $80-20$ training-testing dataset split is utilized.  
All trajectories are represented in the coordinate chart where the learning takes place. Ground-truth positions are mapped to the manifold via the corresponding embedding function when computing evaluation metrics and prediction errors. 


\section{Results}\label{sec:results}
\subsection{Experimental Setup}

In the following experiments, we evaluate a GRU model [details in Apdx. \ref{appendix:model_choice}], with a hidden dimension of 32. We compare models trained with the proposed custom loss functions, $\text{exp}\_{10}$ and $\text{mixed}\_{7\_7}$, against a model trained with the standard loss function $\mathcal{L}_{50}$, as well as a baseline vector addition model defined in Eq. \eqref{eqn:baseline_model}. The loss functions rely on subsampled indices (see Table \ref{tab:subsampling_explicit}), which determine the evaluation points used during training. For each configuration, 10 independent training runs are performed, and reported performance metrics correspond to averages (as well as standard deviations) of the evaluation metrics across these runs. 

To compare the performance of the models, we will consider the mean squared error (MSE) of their predictions versus the ground truth. In Figs. \ref{fig:experiment1}, \ref{fig:missing_velo_v1}, and \ref{fig:missing_velo_v2}, we are reporting the MSE averaged across the $N=10$ independently trained models for each time step, namely 
\begin{equation}\label{eqn:timestep_MSE}
    \bar{e}_{t} = \frac{1}{N} \sum_{i=1}^{N} \| \varphi(\hat{p}_{i,t}) - \varphi(p_{i,t}) \|_{2}^{2}.
\end{equation}
In Tables \ref{tab:diff_setups} and \ref{tab:recon_surface}, we are reporting the MSE averaged across both the 10 independently trained models and all time steps in the trajectory to create a summary statistic, given by
\begin{equation}\label{eqn:summary_MSE}
    \bar{e} = \frac{1}{n\cdot N} \sum_{i=1,t=1}^{N,n} \| \varphi(\hat{p}_{i,t}) - \varphi(p_{i,t}) \|_{2}^{2},
\end{equation}
\textcolor{black}{where $n=100$ indicates the number of trajectory time steps.}
\begin{table}
    \centering
    \caption{The different loss functions and the corresponding indices aggregated.}
    \label{tab:subsampling_explicit}
    \begin{tabular}{ll}
\toprule
      Scheme & Indices \\
\midrule
     $\mathcal{L}_{50}$ & 1,\ldots, 50 \\
   exp\_10 & 1, 1, 2, 2, 4, 5, 8, 12, 17, 24, 35, 50 \\
mixed\_7\_7 & 1, 2, 3, 4, 5, 6, 7, 14, 20, 26, 32, 38, 44, 50 \\
\bottomrule
\end{tabular}
\label{tab:indices}
    
\end{table}

\subsection{Baseline Model}

On an arbitrary manifold, straight lines in the coordinate chart do not correspond to geodesics on the manifold. That is, in general
\begin{equation}\label{eqn:line_not_geodesics}
    \text{exp}_{p}(v) \neq p + v,
\end{equation}
unless the manifold is Euclidean.
However, vector addition in Eq. \eqref{eqn:line_not_geodesics} does provide a first-order approximation of the exponential map. We define a baseline model that performs simple vector addition in the ambient space: 
\begin{equation}\label{eqn:baseline_model}
    p_{t} = p_{0} + \sum_{i=1}^{t} v_{i}.
\end{equation}
This baseline model ignores the curvature of $\mathcal{M}$ and the path-dependency of parallel transport, serving as a reference point to evaluate geometry-aware neural network architectures. 

Note that the baseline model performs direct vector addition of velocity vectors without accounting for the underlying manifold structure. As a result, it achieves zero error on the plane but does not generalize to curved manifolds. Since this model is deterministic and non-trainable, it exhibits no variability across runs and therefore has no associated standard deviation in the subsequent numerical results.

    \subsection{Performance of Common Manifolds}
    We compare the performance of GRU models with a hidden dimension of 32 on Brownian motion and AR(1) datasets defined on the plane, sphere, and torus. Each model is trained using the standard loss function $\mathcal{L}_{50}$, as well as the two proposed custom loss functions: (1) $\text{exp}\_{10}$ and (2) $\text{mixed}\_{7}\_{7}$, which are  defined in Tables \ref{tab:subsampling} and \ref{tab:subsampling_explicit}. The average prediction error (MSE) across all trajectory time steps, is presented in Table \ref{tab:diff_setups}. \textcolor{black}{The average prediction error at each step of the trajectory, in mean squared error, is given in Fig. \ref{fig:experiment1}. }
    
    In Table \ref{tab:diff_setups}, we see that $\text{exp}\_{10}$  and $\text{mixed}\_{7}\_{7}$ are generally outperforming the standard loss function $\mathcal{L}_{50}$ on both AR(1) and Brownian motion trajectories on the plane, sphere, and torus. Further, $\text{exp}\_{10}$ and $\text{mixed}\_{7}\_{7}$ are outperforming the baseline model on the sphere and torus AR(1) datasets and torus Brownian motion dataset. In Fig. \ref{fig:experiment1}, we see that the baseline model is generally beaten on the Brownian motion and AR(1) datasets on the sphere and torus. The $\text{mixed}\_7\_7$ loss function generally produces the best performing models across all of the datasets. The results demonstrate the potential of the geometryc-aware learning models in learning the underlying mechanisms of predicting trajectory location.

\begin{table*}[htbp]
    \centering
    \caption{Prediction errors (MSE over all 100 timesteps) of GRUs trained with different loss functions. Values shown are averaged over 10 runs, with standard deviation in parentheses. Bold signifies the lowest MSE, while an asterisk next to the result indicates the lowest recurrent-based MSE.}
    \label{tab:diff_setups}
    
    \begin{tabular}{lcccccc}
        \toprule
        \multicolumn{1}{c}{\textbf{}}           & \multicolumn{2}{c}{\textbf{Plane}}    & \multicolumn{2}{c}{\textbf{Sphere}}   & \multicolumn{2}{c}{\textbf{Torus}}    \\ \cmidrule(lr){2-3} \cmidrule(lr){4-5} \cmidrule(lr){6-7}
        \multicolumn{1}{c}{\textbf{Loss Func.}} & \textbf{Brownian}    & \textbf{AR(1)}    & \textbf{Brownian} & \textbf{AR(1)}    & \textbf{Brownian} & \textbf{AR(1)}    \\ \hline
        Baseline                                & -- (--)      & -- (--)            & \textbf{0.00005 (--)}      & 0.00200 (--)      & 0.00387 (--)      & 0.00742 (--)      \\ \hline
        $\mathcal{L}_{50}$                      & 0.00118 (0.00109) & 0.00054 (0.00037) & 0.00007 (0.00002) & 0.00018 (0.00030) & 0.00161 (0.00106) & 0.00055 (0.00038) \\
        exp\_10                                 & 0.00030 (0.00048) & 0.00009 (0.00006) & 0.00009 (0.00012) & \textbf{0.00013 (0.00010)} & 0.00054 (0.00033) & \textbf{0.00025 (0.00015)} \\
        mixed\_7\_7                             & \textbf{0.00027 (0.00022)} & \textbf{0.00007 (0.00002)} & $0.00006 \,(0.00003)^{*}$ & 0.00013 (0.00012) & \textbf{0.00029 (0.00018)} & 0.00046 (0.00018)\\
        \bottomrule
    \end{tabular}
    \end{table*}

\begin{figure*}[htbp]
    \centering
    \includegraphics[scale=0.4]{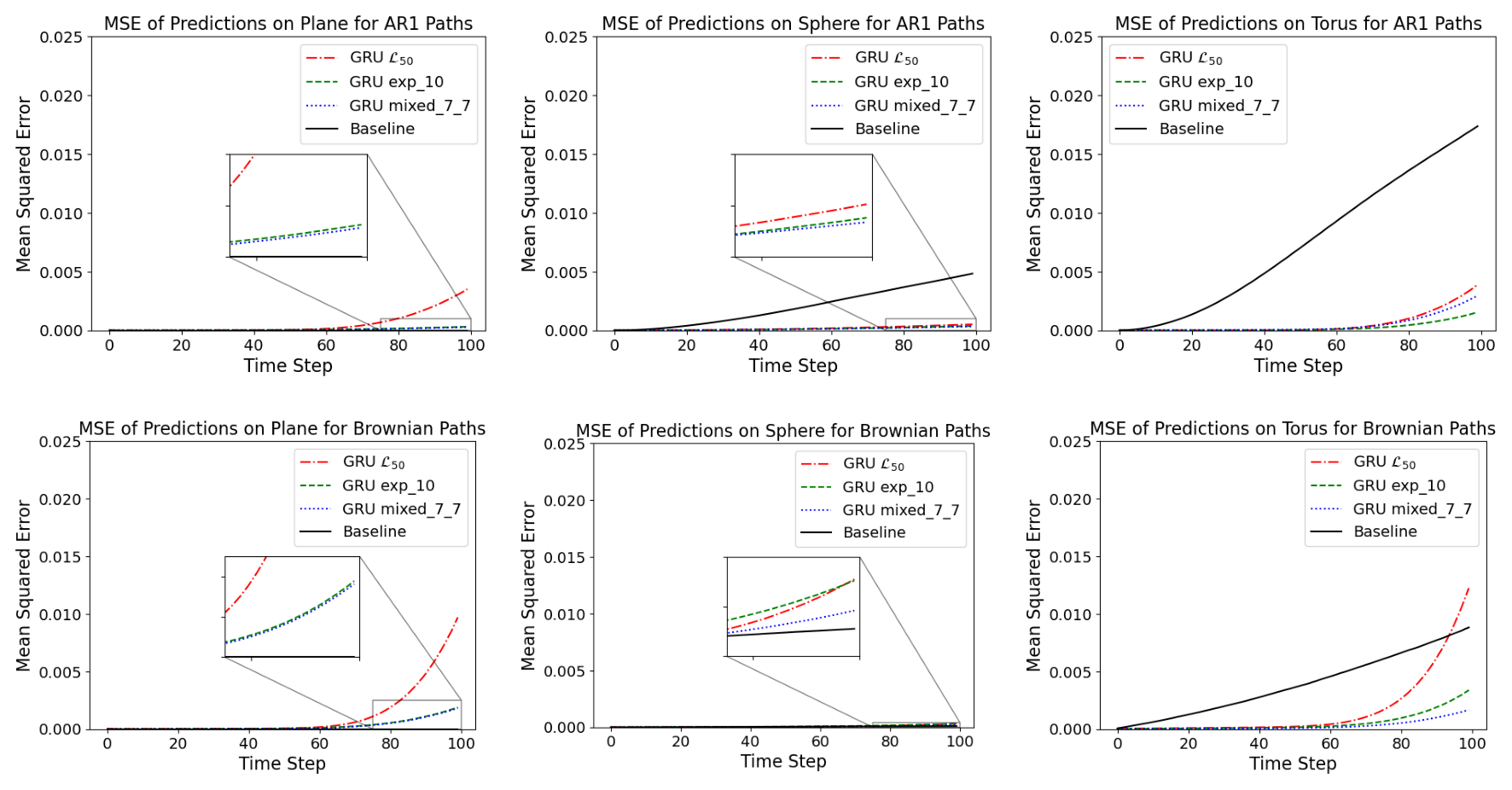}
    \caption{Mean squared error at each time step in the test trajectories for GRU models with different loss functions averaged across 10 independently trained models of each respective type against the vector addition baseline model.}
    \label{fig:experiment1}
\end{figure*}




    \subsection{Performance on the Reconstructed Mountain Surface}
    
    
    To demonstrate the learning advantages of the proposed architecture over the Euclidean baseline approach in more practical situations, we compare the performance of GRU models with a hidden dimension of 32 on Brownian motion and AR(1) datasets generated on the reconstructed mountain surface obtained from the LiDAR point clouds \cite{san_augustin_data}. Each model is trained using the standard loss function $\mathcal{L}_{50}$, and the two proposed custom loss functions: (1) $\text{exp}\_{10}$ and (2) $\text{mixed}\_{7}\_{7}$, as defined in Tables \ref{tab:subsampling} and \ref{tab:subsampling_explicit}. The average prediction error (MSE) across all trajectory time steps, is presented in Table \ref{tab:recon_surface}. In Table \ref{tab:recon_surface}, the $\text{mixed}\_7\_7$ custom loss function has superior performance on both the AR(1) and Brownian motion datasets on the reconstructed mountain surface. \textcolor{black}{Additioanlly, supplemental results are provided in Table \ref{tab:model_choice_supp_mixed77} in Apdx. \ref{appendix:model_choice} showing the comparison of RNN, LSTM, and GRU with the $\text{mixed}\_7\_7$ loss function on the reconstructed mountain surface showing the learning edge GRUs can provide over other recurrent architectures.}
    
\begin{table}[htbp]
    \centering
    \caption{Prediction MSE on the reconstructed surface}
    \label{tab:recon_surface}
    \begin{tabular}{lcc}
        \toprule

        \textbf{Setup}  & \textbf{Brownian} & \textbf{AR(1)}\\
        
        \midrule
        Baseline  & 0.01268 (NA) & 0.00225 (NA)\\
        \midrule
        $\mathcal{L}_{50}$  & 0.00291 (0.00112) & 0.00032 (0.00011)\\
        exp\_10  & 0.00276 (0.00102) & 0.00032 (0.00014)\\
        mixed\_7\_7  & \textbf{0.00271 (0.00095)} & \textbf{0.00031 (0.00011)}\\
        \bottomrule
    \end{tabular}
\end{table}
    
\subsection{Handling Missing Velocities}
In practical scenarios, velocity sensors in moving vehicles may fail intermittently due to outages or evasion in which case they are incapable of returning measurements.  We study the proposed learning architecture's ability to make predictions when there are velocities missing from the input sequence and propose remedies to improve the learning performance in the presence of missing data. 

    \subsubsection{Fixed Time-Step Velocity Dropout}
    We evaluate the model's ability to make predictions when a single velocity is missing from the input sequence. Specifically, the datasets are modified such that the velocity at the $10^{th}$ timestep is missing from every test trajectory and replaced with a zero. The model prediction MSE for  Brownian motion trajectories on the torus is shown in Fig. \ref{fig:missing_velo_v1}. From Fig. \ref{fig:missing_velo_v1}, we see that there is a notable spike in the  MSE of all considered models at the $10^{th}$ time step, as expected. Beyond the $10^{th}$ time step, the three GRU models are consistently outperforming the baseline model, and the custom loss functions, $\text{exp}\_{10}$ and $\text{mixed}\_{7}\_{7}$, have the best performance while offering a more 'graceful degradation'. Clearly, the proposed learning architecture exhibits more robustness compared to the baseline approach, while the proposed loss functions outperform the loss including all the first $50$ training trajectory terms. The performance advantage is more pronounced in the reconstructed surface, showing the potential of the learning architecture even in situations where the manifold does not have an analytical characterization.


\begin{figure*}[htbp]
    \centering
    \begin{subfigure}[b]{0.4\textwidth}
        \includegraphics[width=\textwidth]{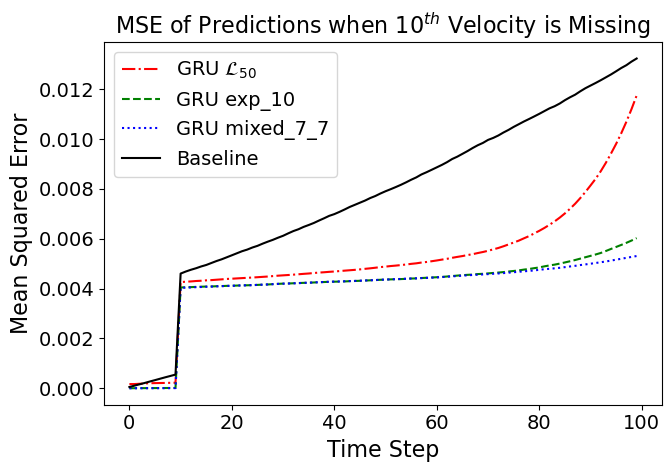}
        \caption{Torus.}
        \label{fig:missing_velo_v1_torus}
    \end{subfigure}
    \begin{subfigure}[b]{0.4\textwidth}
        \includegraphics[width=\textwidth]{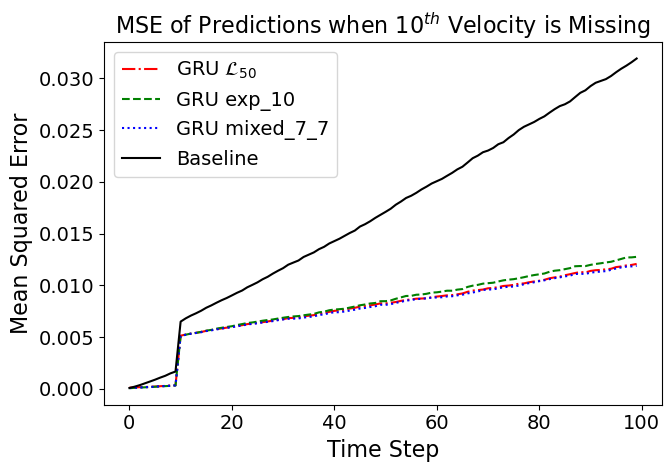}
        \caption{Reconstructed surface.}
        \label{fig:missing_velo_v1_reconsurf}
    \end{subfigure}
    \caption{Predicton MSE for Brownian motion trajectories on the torus and the reconstructed surface with the $10^{th}$ time step of the velocity sequence missing.}
    \label{fig:missing_velo_v1}
\end{figure*}

    \subsubsection{Random Velocity Masking}
  Building on the previous experimental scenario,  we evaluate the model under more severe missing-data conditions by randomly masking a proportion of velocity observations via the utilization of a Bernoulli process. We consider three missing velocity proportions: 1\%, 5\%, and 10\%. The model prediction MSE for  Brownian motion trajectories on the torus and the reconstructed surface are shown in Figure \ref{fig:missing_velo_v2}. From Fig. \ref{fig:missing_velo_v2}, we see clearly that the GRU models are consistently beating the baseline model. The custom loss functions, $\text{exp}\_{10}$ and $\text{mixed}\_{7}\_{7}$, show an advantage over the standard loss function, $\mathcal{L}_{50}$ after approximately the $75^{th}$ time step showing that the custom loss functions provide better long term predictions in adverse data environments involving missing data. It is worth emphasizing that the proposed learning architectures experience a more graceful degradation compared to the baseline as the portion of missing velocities increases, making it an appealing option in situations where robustness is required. Moreover the custom loss functions exhibit a more linear behavior with respect to ${\cal L}_{50}$ which keeps the MSE performance consistently below the baseline MSE.


\begin{figure*}[h]
    \centering
    \begin{subfigure}[b]{0.85\textwidth}
        \centering
        \includegraphics[width=\textwidth]{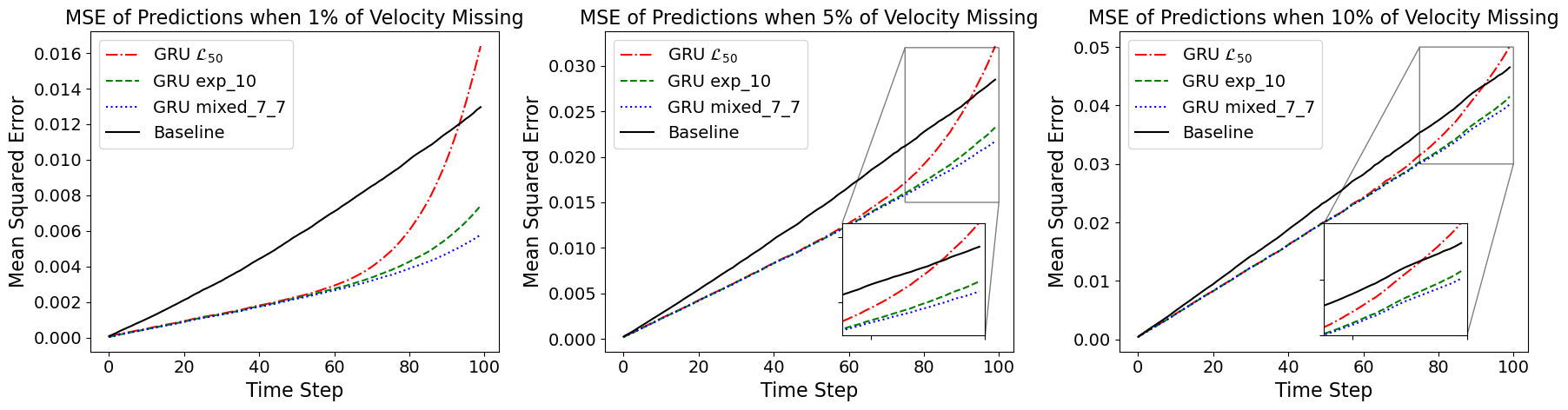}
        \caption{Torus.}
        \label{fig:missing_velo_v2_torus}
    \end{subfigure}
    
    \begin{subfigure}[b]{0.85\textwidth}
        \centering
        \includegraphics[width=\textwidth]{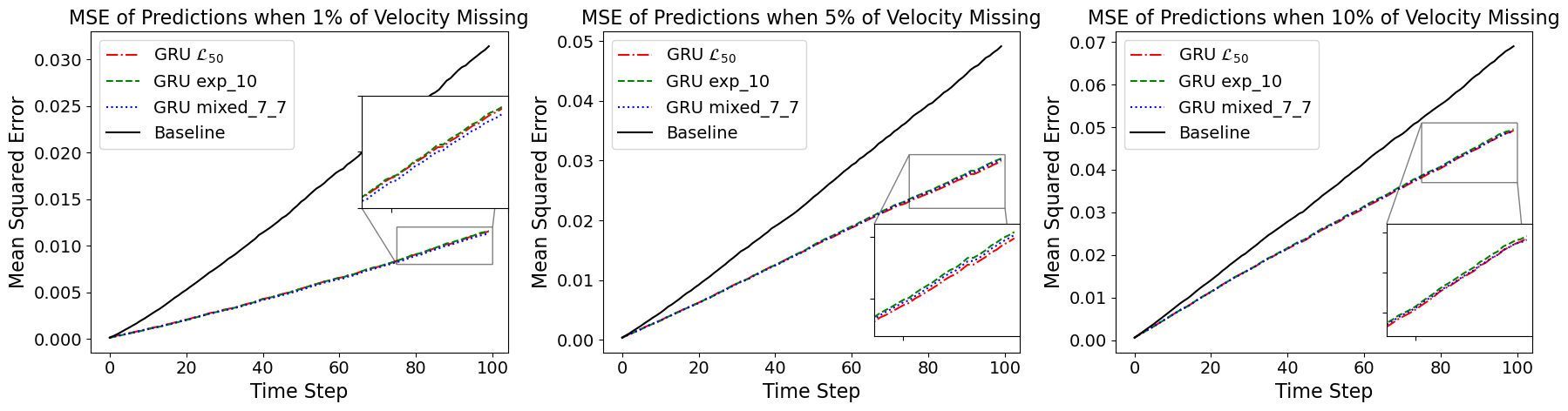}
        \caption{Reconstructed surface.}
        \label{fig:missing_velo_v2_reconsurf}
    \end{subfigure}
    \caption{Mean squared error of predictions for Brownian motion trajectories on the torus and the reconstructed surface for different missing velocity portions.}
    \label{fig:missing_velo_v2}
\end{figure*}
\section{Conclusions and Future Work}\label{sec:conclusion}

In this paper, we proposed a GRU-based learning architecture operating on coordinate charts that can effectively perform navigation and position prediction on nonlinear manifolds whose geometric structure is not necessarily known. Across curved manifolds, including the sphere, torus, and a reconstructed mountain surface, GRU-based approaches consistently outperform a naive vector-addition baseline that ignores the underlying surface geometry. Furthermore, we  devised pertinent loss functions, i.e., $\text{exp}\_{10}$ and $\text{mixed}\_{7\_7}$, to strategically sample the training trajectories and effectively train on manifold trajectories yielding improved performance over the standard loss function $\mathcal{L}_{50}$, particularly in regimes where geometric consistency over time is critical.

Importantly, these advantages become more pronounced under adverse data conditions where data is missing from the velocity observations, where the proposed loss functions enable more stable and accurate long-horizon predictions compared to both the baseline model and GRU models trained with standard loss functions that take into account all training trajectory positions. This suggests that incorporating structure-aware loss formulations plays a key role in mitigating error accumulation in sequential prediction on manifolds.

Our findings further advocate that the proposed approach is well-suited for applications requiring reliable long-term trajectory estimation in partially observed or degraded sensing environments. In scenarios involving signal denial or radar evasion, where measurements may be sparse, intermittent, or intentionally disrupted, the ability to maintain accurate predictions over extended time horizons is essential. The demonstrated robustness of the proposed loss functions highlights their potential for such settings.

Future work will explore extending these methods to higher-dimensional and more complex manifolds, as well as incorporating uncertainty-aware modeling.


\appendices

\section{Comparing Recurrent Architectures}\label{appendix:model_choice}
In Table \ref{tab:model_choice} we compare, across $10$ independent trials and across all time steps, the prediction MSE of the vector addition baseline model against the MSE attained by a gated recurrent unit (GRU), a long short-term memory (LSTM) unit, and an Elman recurrent neural network (RNN) when trained with different versions of the standard loss function, $\mathcal{L}_{10}$, $\mathcal{L}_{30}$, and $\mathcal{L}_{50}$, where the number of added position terms changes and indicated by the loss subscript. Each of the GRU, LSTM, and RNN learning models has a hidden dimension of 32. Each model is trained across 10 independent trials.  Table \ref{tab:model_choice} shows that the $\mathcal{L}_{50}$ version of the standard loss function notably outperforms $\mathcal{L}_{10}$ and $\mathcal{L}_{30}$. Thus, for the experiments in Section \ref{sec:results} we utilize the $\mathcal{L}_{50}$ version of the standard loss function. 

In Tables \ref{tab:model_choice_supp} and \ref{tab:model_choice_supp_mixed77}, we compare the prediction MSE of the three different recurrent architectures on the reconstructed sufrace trajectories for the loss functions   $\mathcal{L}_{50}$ and 
$\text{mixed}\_{7\_7}$, respectively. When the $\mathcal{L}_{50}$, LSTMs achieve  lower MSE for Brownia motion, whereas GRUs have an edge for AR(1) trajectories. It can also be seen that the GRU produces a lower average prediction error than the LSTM and RNN on the reconstructed surface when using the $\text{mixed}\_{7\_7}$ loss function in training. Additionally, since the gating mechanism in the GRU model is computationally simpler than that of the LSTM \cite{dey2017gatevariantsgatedrecurrentunit}, moving forward, we will utilize the GRU model.

\begin{table*}[]
    \centering
    \caption{Prediction error (MSE) for GRU, LSTM, and RNN with different versions of the standard loss function. The mean values and standard deviation (in parentheses) of the localization MSE are displayed.}
    \label{tab:model_choice}
\footnotesize
\begin{tabular}{cccccccc}
\hline
                                    & \multicolumn{1}{l}{}               & \multicolumn{2}{c}{\textbf{Plane}}    & \multicolumn{2}{c}{\textbf{Sphere}}   & \multicolumn{2}{c}{\textbf{Torus}}    \\ \cline{3-8} 
\textbf{\begin{tabular}[c]{@{}c@{}}Loss \\ Func.\end{tabular}}                 & \multicolumn{1}{l}{\textbf{Model}} & \textbf{Brownian} & \textbf{AR(1)}    & \textbf{Brownian} & \textbf{AR(1)}    & \textbf{Brownian} & \textbf{AR(1)}    \\ \hline
                                    & Baseline                           & -- (--)      & -- (--)      & \textbf{0.00005 (--)}      & 0.00200 (--)      & 0.00387 (--)      & 0.00742 (--)      \\ \hline
\multirow{3}{*}{$\mathcal{L}_{10}$} & GRU                                & 0.10376 (0.09999) & 0.28024 (0.08625) & 0.03291 (0.00406) & 0.02627 (0.01133) & 0.11175 (0.02301) & 0.09435 (0.01313) \\
                                    & LSTM                               & 0.54111 (0.32620) & 0.06560 (0.02796) & 0.00936 (0.00489) & 0.01819 (0.00631) & 0.14782 (0.05111) & 0.51796 (0.13962) \\
                                    & RNN                                & 2.26389 (6.02657) & 0.00985 (0.02976) & 0.00150 (0.00171) & 0.00811 (0.00604) & 0.08622 (0.16793) & 0.10722 (0.27902) \\ \hline
\multirow{3}{*}{$\mathcal{L}_{30}$} & GRU                                & 0.00574 (0.00411) & 0.00608 (0.00385) & 0.00149 (0.00041) & 0.00019 (0.00017) & 0.00334 (0.00155) & 0.01392 (0.00225) \\
                                    & LSTM                               & 0.00146 (0.00158) & 0.00052 (0.00111) & 0.00086 (0.00061) & 0.00058 (0.00043) & 0.00521 (0.00111) & 0.02413 (0.05700) \\
                                    & RNN                                & 0.00081 (0.00090) & 0.00154 (0.00155) & 0.00841 (0.01158) & 0.00664 (0.00675) & 0.55926 (0.77857) & 0.56647 (0.63164) \\ \hline
\multirow{3}{*}{$\mathcal{L}_{50}$} & GRU                                & 0.00118 (0.00109) & 0.00054 (0.00037) & $0.00007 \, (0.00002)^{*}$ & \textbf{0.00018 (0.00030)} & 0.00161 (0.00106) & \textbf{0.00055 (0.00038)} \\
                                    & LSTM                               & \textbf{0.00030 (0.00021)} & \textbf{0.00016 (0.00011)} & 0.00205 (0.00389) & 0.00058 (0.00047) & \textbf{0.00070 (0.00055)} & 0.00229 (0.00218) \\ 
                                    & RNN                               & 0.16437 (0.51187) & 0.22246 (0.39899) & 0.00699 (0.01420) & 0.00587 (0.00563) & 0.21869 (0.63981) & 0.59643 (0.65239)  \\ \hline
\end{tabular}
\end{table*}

\begin{table}[htbp]
    \centering
    \caption{Prediction error (MSE) of recurrent models on the reconstructed surface with the $\mathcal{L}_{50}$ loss function.}
    \label{tab:model_choice_supp}
    \begin{tabular}{lcc}
        \toprule

        \textbf{Setup}  & \textbf{Brownian} & \textbf{AR(1)}\\
        
        \midrule
        Baseline  & 0.01268 (NA) & 0.00225 (NA)\\
        \midrule
        GRU & 0.00291 (0.00112) & \textbf{0.00032 (0.00011)} \\
        LSTM  & \textbf{0.00264 (0.00094)} & 0.00036 (0.00010)\\
        RNN  & 0.01611 (0.02645) & 	0.00112 (0.00034)\\
        \bottomrule
    \end{tabular}
\end{table}

\begin{table}[htbp]
    \centering
    \caption{Prediction error (MSE) of recurrent models on reconstructed surface with the $\text{mixed}\_{7\_7}$ loss function.}
    \label{tab:model_choice_supp_mixed77}
    \begin{tabular}{lcc}
        \toprule

        \textbf{Setup}  & \textbf{Brownian} & \textbf{AR(1)}\\
        
        \midrule
        Baseline  & 0.01268 (NA) & 0.00225 (NA)\\
        \midrule
        GRU & \textbf{0.00271 (0.00095)} & \textbf{0.00031 (0.00011)}  \\
        LSTM  & 0.00295 (0.00098) & 0.00036 (0.00011)\\
        RNN  & 0.06543 (0.07541) & 	0.01220 (0.02937)\\
        \bottomrule
    \end{tabular}
\end{table}

\section{Manifold Distance}\label{appendix:rieman_submanifold}

Consider a submanifold $\mathcal{M}^{m} \subset \mathbb{R}^n$, $m\leq n$, parameterized by a smooth embedding $\varphi:U \to \mathbb{R}^{n}$ where $U$ is an open set. We call $\varphi$ the embedding function and $U$ the coordinate chart. Suppose that the Jacobian of the embedding function, $J_{\varphi}(x) \in \mathbb{R}^{n\times m}$, is full rank, that is, $rank(J_{\varphi}(x)) = m$ for all $x \in U$. Then, for any point $x \in U $, with $\varphi(x) = p$, an inner product on the tangent space $T_{x}U$ is defined by 
\begin{equation}\label{eqn:induced_metric}
    g_{x}(u,v) = \langle J_{\varphi}(x)u, J_{\varphi}(x)v\rangle_{\mathbb{R}^{n}}, \quad u,v \in T_{x}U
\end{equation}
where the right-hand side is the Euclidean inner product in $\mathbb{R}^{n}$. The collection $g = \{g_{x}: x\in U \}$ is the pullback metric on $U$ \cite{lee2018introduction}. The length of a path $\gamma : [0,1] \to U$ under metric $g$ is 
\begin{equation}
    L(\gamma) \coloneqq \int_0^1 \sqrt{g_{\gamma(t)}\!\left(\gamma'(t),\, \gamma'(t)\right)}\, dt,
\end{equation}
where $\gamma \in C^{1}$ (continuously differentiable) and $\gamma '(t)$ is the derivative with respect to $t$ and represents the velocity of the parameterized curve at time $t$. Using the definition of the inner product on the tangent space $T_{x}U$, from Equation \eqref{eqn:induced_metric},
\begin{align}
    g_{\gamma(t)}(\gamma ' (t), \gamma ' (t)) &= \langle J_{\varphi}(\gamma(t))\gamma'(t), J_{\varphi}(\gamma(t))\gamma'(t) \rangle_{\mathbb{R}_{n}} \\
     &= \| J_{\varphi} (\gamma(t))\gamma ' (t) \|_{2}^{2}.
\end{align}
Using the chain rule, $(\varphi \circ \gamma)'(t) = d\varphi_{\gamma (t)}(\gamma'(t)) = J_{\varphi}(\gamma(t))\gamma'(t)$, where $d\varphi_{\gamma(t)}$ denotes the differential and its matrix representation in coordinates is the Jacobian.
Thus,
\begin{equation}
    L(\gamma) = \int_0^1 \left\|(\varphi \circ \gamma)'(t)\right\|_2 dt,
\end{equation}
where the second equality follows from the definition of the pullback metric and shows that $L(\gamma)$ coincides with the Euclidean arc length of the embedded curve $\varphi \circ \gamma$ whose image lies in $\mathcal{M}$.

The distance between two points $p,q \in \mathcal{M}$ is defined as 
\begin{equation} \label{eqn:dist_metric}
    d_{\mathcal{M}}(p, q) \coloneqq \inf_{\gamma} \left\{ L(\gamma) \right\},
\end{equation}
where the infimum is taken over twice differentiable paths $\gamma$ with endpoints $(\varphi \circ \gamma) (0) = p$ and  $(\varphi \circ \gamma) (1) = q$.

\section{Proof of Theorem 1}\label{appendix:proof}

\begin{proof}
    To establish the left bound, consider the path function $\gamma:[0,1] \to U$ be a $C^{1}$ curve satisfying $(\varphi \circ \gamma)(0)=p$ and $(\varphi \circ \gamma)(1)=q$ [cf. Apdx. B]. Then, by the Fundamental Theorem of Calculus \cite{rudin1976principles} and the triangle inequality,

    \begin{align*}
        \|p-q\|_{2} &= \left \| \int_0^1 (\varphi \circ \gamma)'(t)dt \right\|_{2} \\
        &\le \int_0^1\left \| (\varphi \circ \gamma)'(t) \right\|_{2} dt  = L(\gamma).
    \end{align*}
    Since this holds for every such $\gamma$, taking the infimum over all $C^{1}$ curves that satisfy $(\varphi \circ \gamma)(0) = p$ and $(\varphi \circ \gamma)(1) = q$, gives 
    \begin{equation}
        \|p-q\| \leq \text{inf}_{\gamma}\{ L(\gamma)\} = d_{\mathcal{M}}(p,q),
    \end{equation}
where the last equality is from the definition provided in Eq. \eqref{eqn:dist_metric}.

To establish the right bound, we will reparameterize the path function as a function of the arc (path) length. Specifically,  let $\Gamma: [0, \ell] \to \mathcal{M}$ be the minimizing unit-speed geodesic joining $p$ to $q$, where $\ell = d_{\mathcal{M}}(p,q)$ parameterized with respect to arc length $s\in[0,\ell]$, i.e., $s(0)=0$ and $s(1)=\ell$. Consequently, $\|\Gamma'(s)\| = 1$ for all $0 \leq s \leq \ell$. 

Since $\Gamma$ is parameterized by arc length, the Frenet-Serret formulas \cite{doCarmo1976} imply, $\|\Gamma''(s)\| = |\kappa(s)| \leq K$, as assumed in Thm. 1. Writing $\Gamma'(s) = \Gamma'(0) + \int_0^s \Gamma''(\sigma)\, d\sigma$ and integrating,
\begin{align}
    \|p - q\| &= \left\|\int_0^\ell \Gamma'(t)\, dt\right\| \\
    &= \left\|\ell\, \Gamma'(0) + \int_0^\ell \int_0^s \Gamma''(\sigma)\, d\sigma\, ds\right\| \\
    &\geq \ell\|\Gamma'(0)\| - \int_0^\ell \int_0^s \|\Gamma''(\sigma)\|\, d\sigma\, ds \\
    &\geq \ell - K\int_0^\ell s\, ds = \ell - \frac{K\ell^2}{2} \\
    &= \ell\left(1 - \frac{K\ell}{2}\right).
\end{align}
Setting $r = 1/K$, and under the assumption $d_{\mathcal{M}}(p,q) < r$ which is reasonable since the velocity is bounded and the path length $d_{\mathcal{M}}(p,q)$ can be made arbitrarily small for traveling within one unit of time. Thus, we have $1 - K\ell/2 > 1/2$, and so
\begin{equation}
    \|p - q\| \geq \frac{d_{\mathcal{M}}(p,q)}{2},
\end{equation}
which rearranges to $d_{\mathcal{M}}(p,q) \leq 2\|p-q\|$.

Therefore, taking $r = 1/K$ and $C = 2$, both inequalities hold for all $p, q \in \mathcal{M}$ with $d_{\mathcal{M}}(p,q) < r$.\end{proof}



\bibliographystyle{IEEEtran}
\bibliography{refernces}

\end{document}